\documentclass[12 pt]{article}
\usepackage{fullpage,graphicx,psfrag,verbatim}
\usepackage[small,bf]{caption}
\usepackage{mytex}
\usepackage{amsmath, amssymb,graphicx,hyperref, bbm}
\usepackage{enumitem}
\usepackage{epstopdf}
\usepackage[numbers]{natbib}
\usepackage{algorithm}
\usepackage{algpseudocode}
\usepackage{amsfonts}
\usepackage{amssymb}
\usepackage{hyperref}
\usepackage{tikz}
\usepackage{verbatim}
\usepackage{xcolor}
\usepackage{booktabs}
\newcommand{\Var}{\mathrm{Var}}
\newcommand{\semi} {\mathbb{S}}

\newcommand{\Real} {\mathbb{R}}
\newcommand{\Corr}{\mathrm{Corr}}
\newcommand{\Cov}{\mathrm{Cov}}
\newcommand{\trace}{\textbf{trace}}
\newcommand{\Ind}{\mathbbm{1}}

\newcommand{\identity}{\mathbf{I}}
\newcommand{\prox}{\mathbf{prox}}

\newcommand{\Argmin}{\text{argmin}}
\usepackage{color, colortbl}
\definecolor{Gray}{gray}{0.9}
\definecolor{LightCyan}{rgb}{0.88,1,1}
\definecolor{LightBlue}{rgb}{0,100,100}
\definecolor{LightYellow}{rgb}{100,100,0}
\definecolor{LightRed}{rgb}{100,0,0}
\begin{document}
\title{Sparse canonical correlation analysis}
\author{Xiaotong Suo \thanks{Institute for Computational and Mathematical Engineering,  Stanford University. Corresponding author: xiaotong@stanford.edu.}, Victor Minden \thanks{Institute for Computational and Mathematical Engineering, Stanford University.}, Bradley Nelson \thanks{Institute for Computational and Mathematical Engineering, Stanford university.}, \\ Robert Tibshirani \thanks{Departments of Biomedical Data Sciences, and Statistics, Stanford University}, Michael Saunders\thanks{Department of Management Science and Engineering and Institute for Computational and Mathematical Engineering, Stanford University}}
\maketitle
\begin{abstract}
Canonical correlation analysis was proposed by \citet{HOTELLING01121936} and it measures linear relationship between two multidimensional variables. In high dimensional setting, the classical canonical correlation analysis breaks down. We propose a sparse canonical correlation analysis by adding $\ell_1$ constraints on the canonical vectors and show how to solve it efficiently using linearized alternating direction method of multipliers (ADMM) and using TFOCS as a black box. We illustrate this idea on simulated data. 
\end{abstract}
\section{Introduction}
Correlation measures dependence between two or more random variables. The most popular measure is the Pearson's correlation coefficient. For random variables $x, y \in \reals$, the population correlation coefficient is defined as $\rho_{x,y} = \frac{\cov (x, y)}{\sqrt{\var(x)}\sqrt{\var(y)}}$. It is of importance that the correlation takes out the variance in random variables $x$ and $y$ by dividing the standard deviation of them. We could not emphasize more the importance of this standardization, and we present two toy examples in \autoref{covariancetable}. Clearly, $x$ and $y$ are more correlated in the left table than the right table even though the covariance between $x$ and $y$ are seemingly much smaller in the left table than that in the right table. 
\begin{table}
\centering
\begin{tabular}{l | c | c |c}
Covariance & $x$ & $y$  \\
\hline 
$x$ &0.1 & 0.09 &  \\ 
$y$ & 0.09 & 0.1 & 
\end{tabular}
\quad \quad \text{and}\quad \quad
\begin{tabular}{l | c | c |c}
Covariance & $x$ & $y$  \\
\hline 
$x$ & 0.9 & 0.3&  \\ 
$y$ & 0.3 & 0.9 & 
\end{tabular}
\caption{Covariance Matrix}
\label{covariancetable}
\end{table}

Canonical correlation studies correlation between two multidimensional random variables.  Let $x \in \reals ^p$ and $y \in \reals^q$ be random variables, and let $\Sigma_x$, $\Sigma_y$ be covariance of $x$ and $y$ respectively, and their covariance matrix be $\Sigma_{xy}$.  In simple words, it seeks linear combinations of $x$ and $y$ such that the resulting values are mostly correlated. The mathematical definition is
\begin{align}
\underset{u\in \reals^p , v\in\reals^q}{\text{maximize}} \ \ \ &  \frac{u^T \Sigma_{xy}v}{\sqrt{u' \Sigma_{x} u}\sqrt{v' \Sigma_y v)}}. 
\label{realone}
\end{align}
Solving \autoref{realone} is easy in low dimensional setting, i.e., $n \gg p$, because we can use change of variables: $\Sigma^{1/2}_{x}u = a $, and $\Sigma_y^{1/2}v = b$.  \autoref{realone} becomes 
\begin{align}
\underset{a\in \reals^p, b\in \reals^q}{\text{maximize}} \ \ \ &  \frac{a^T\Sigma_x^{-1/2}\Sigma_{xy}\Sigma_y^{-1/2}b}{\sqrt{a^Ta}\sqrt{b^Tb}}.
\label{svd}
\end{align}

Solving \autoref{svd} is equivalent to solving singular decomposition of the new matrix $\Sigma^{-1/2}_x\Sigma_{xy}\Sigma_y^{-1/2}$.  However, when $p \gg n$, this method is not feasible because $\Sigma_x^{-1}, \Sigma_y^{-1}$ can not be estimated accurately.  Moreover, we might want to seek a sparse representation of features in $x$ and features in $y$ so that we can get interpretability of the data. 

Let $X\in \reals^{n \times p }, Y \in \reals^{n \times q}$ be the data matrix. We consider a regularized version of the problem
\begin{align*}
\underset{u,v}{\text{minimize}} \ \ \ & -\Cov(Xu, Yv) + \tau_1|u|_{1}  + \tau_2|v|_1\\
\text{subject to}\ \ \ & \Var(Xu) = 1; \Var(Yv) = 1, 
\label{exact}
\end{align*}
and since the constraints of minimization problem are not convex, we further relax it as

\begin{equation}
\begin{array}{ll}
    \mbox{minimize}   & -\Cov(Xu, Yv) + \tau_1|u|_{1}  + \tau_2|v|_1\\
    \mbox{subject to} &  \Var(Xu) \leq 1; \Var(Yv) \leq 1.
    \end{array}
    \label{biconvex}
\end{equation}
Note that resulting problem is still nonconvex, however, it is a biconvex.

\paragraph{Related Work} Though some research has been done in canonical correlation analysis in high dimensional setting, there are issues we would like to point out: 
\begin{enumerate}
\item Computationally efficient algorithms.  To our best knowledge, we have not found an  algorithm which can be scaled efficiently to solve \autoref{biconvex}. 
\item Correct relaxations. An efficient algorithm to find sparse canonical vectors was proposed by \citet{Witten} but we think the relaxation of $\Var(Xu)=1$ to $\Var(u) = 1$, $\Var(Yv) = 1$ to $\Var(v)=1$ are not very realistic in high dimensional setting. Our algorithms relax the $\Var(Xu)=1$ to $\Var(Xu)\leq 1$, and $\Var(Yv)=1$ to $\Var(v)=1$. Though we can not guarantee the solutions are on the boundary, it is often the case. 
\item Simulated Examples. We consider a variety of simulated examples, including those which are heavily considered in the literature. We also presented some examples which are not considered in the literature but we think their structures are closer to structures of a real data set. 
\end{enumerate}

The paper is organized as follows. Section \ref{sec: main sec} contains motivations and algorithms for solving the first sparse canonical vectors. Subsection \ref{sub: remaining canonical vectors} contains an algorithm to find $r$th canonical vectors, though we only focus on estimating the first pair of canonical vectors in this paper. We show solving sparse canonical vectors is equivalent to solving sparse principle component analysis in a special case in Section \ref{sec: a special case}. We demonstrate the usage of such algorithms on simulated data in Section \ref{sec: simulated data} and show a detailed comparisons among sparse CCA proposed by \citet{sparseCCA}, \citet{Witten}, and \citet{sparseig}. Section \ref{sec: discussion} contains some discussion and directions for future work. 

\section{Sparse Canonical Correlation Analysis}
\label{sec: main sec}

\subsection{The basic idea}
\begin{align*}
\underset{u,v}{\text{minimize}} \ \ \ & -\Cov(Xu, Yv) + \tau_1|u|_{1}  + \tau_2|v|_1\\
\text{subject to}\ \ \ & \Var(Xu) \leq 1; \Var(Yv) \leq 1, 
\end{align*}
This resulting problem is biconvex, i.e., if we fix $u$, the resulting minimization is convex respect to $v$ and if we fix $v$, the minimization is convex respect to $u$:
\begin{enumerate}
\item
	Fix $v$, solve for $u$:
		\begin{align}
		\underset{u}{\text{minimize}} \ \ \ & -\Cov(Xu, Yv) + \tau_1|u|_{1}  +\Ind\{u: \Var(Xu) \leq 1\} 
		\label{solveforu}
		\end{align}
\item
	Fix $u$, solve for $v$:
		\begin{align}
			\underset{v}{\text{minimize}} \ \ \ & -\Cov(Xu, Yv) + \tau_2|v|_1 + \Ind\{v: \Var(Yv) \leq 1\} 
		\label{solveforv}
		\end{align}
\end{enumerate}
In \autoref{algorithmic details}, we describe how to solve the subproblems \autoref{solveforu} and \autoref{solveforv} in details. 

Our formulation is similar to the method proposed by \citet{Witten}. Their formulation is 
\begin{equation}
\begin{array}{ll}
    \mbox{minimize}   & -\Cov(Xu, Yv) + \tau_1|u|_{1}  + \tau_2|v|_1\nonumber\\
    \mbox{subject to} & \|u\|_2 \leq 1;  \|v\|_2 \leq 1.
    \end{array}
\end{equation}
This formulation is obtained by replacing covariance matrices $X^TX$ and $Y^TY$ with identity matrix. They also used alternating minimization approach, and by fixing one of the variable, the other variable has a closed form solution. Their formulation can be solved very efficiently as a result. However, we now present a simple example to show that the solution of their formulation can be very inaccurate and non-sparse. 

\begin{example}
We generate our data as follows: 
\begin{align*}
\begin{pmatrix} X \\ Y \end{pmatrix}\sim N(\begin{pmatrix}
0\\0
\end{pmatrix}, \begin{pmatrix}\Sigma_X &\Sigma_{XY}\\ \Sigma_{YX} & \Sigma_Y\end{pmatrix}), 
\end{align*}
where 
\begin{align*}
(\Sigma_x)_{i,j} = (\Sigma_y)_{i,j} = 0.9 ^{|a-b|},  \Sigma_{XY} = \Sigma_X (u_1 \rho v_1^T) \Sigma_Y, 
\end{align*}
and $u_1$ and $v_1$ are sparse canonical vectors, and the number of non-zero elements are chosen to be 5,  5, respectively. The location of nonzero elements are chosen randomly and normalized with respect to the true covariance of $X$ and $Y$, i.e., $ u_1^T\Sigma_{X}u_1 = 1$ and $v_1\Sigma_{Y}v_1 = 1$.

We first presented a proposition, which was in the paper of \citet{thresholding}: 
\begin{proposition}
\begin{equation}
\begin{array}{ll}
    \mbox{maximize}   & a^T\Sigma_{xy}b\\
    \mbox{subject to} & a^T\Sigma_{x}a = 1; b^T\Sigma_{y}b = 1
    \end{array}
    \label{paper:iterativethresholding}
\end{equation}

When $\Sigma_{xy}$ is of rank 1, the solution (up to sign jointly) of Equation \ref{paper:iterativethresholding} if $(\theta, \eta)$ if and only if the covariance structure between $X$ and $Y$ can be written as 

\[
\Sigma_{xy} = \lambda \Sigma_x \theta \eta^T \Sigma_y,
\]
where $0 \leq \lambda\leq 1$, $\theta ^T \Sigma_x \theta = 1$, and $\eta^T\Sigma_y\eta = 1$. In other words, the correlation between $a^TX $ and $b^TY$ are maximized by $\Corr(\theta^X, \eta^Y)$, and $\lambda$ is the canonical correlation between $X$ and $Y$. 

More generally, the solution of \ref{paper:iterativethresholding} is $(\theta_1, \eta_1)$ if and only if the covariance structure between $X$ and $Y$ can be written  $\Sigma_{xy} = \Sigma_x(\sum_{i = 1}^r \lambda_i \theta_1 \eta_i^T)\Sigma_y$. 

\end{proposition}
The sample size  is $n = 400$, and $p_u = p_v = 800$. We denote their solutions as $\hat u_w, \hat v_w$, and our approach as $\hat u_1, \hat v_1$. We have two main goals when we solve for canonical vectors: maximizing the correlation while maintaining the sparsity in canonical vectors. A common way to measure the performance is to use the Pareto curve, seen in  \autoref{fig:example0} and \autoref{fig:example0witten}. The left panel traces 
\begin{align*}
x: \frac{-\hat{u}^T X^TY \hat{v}}{\sqrt{\hat{u}^TX^TX\hat{u}}\sqrt{\hat{v}Y^TY\hat{v}}} \text{ v.s. } y: \|\hat u\|_{\ell_1} + \|\hat v\|_{\ell_1}, 
\end{align*}
and right panel traces
\begin{align*}
x: \frac{-\hat{u}^T \Sigma_{XY} \hat{v}}{\sqrt{\hat{u}^T\Sigma_X\hat{u}}\sqrt{\hat{v}^T\Sigma_Y\hat{v}}} \text{ v.s. } y: \|\hat u\|_{\ell_1} + \|\hat v\|_{\ell_1}. 
\end{align*}
We prefer a point which is close to the left corner of the Pareto curve, because it represents a solution which consists of sparse canonical vectors and achieves the maximum correlation. 
\begin{figure}
\includegraphics[scale= 0.5]{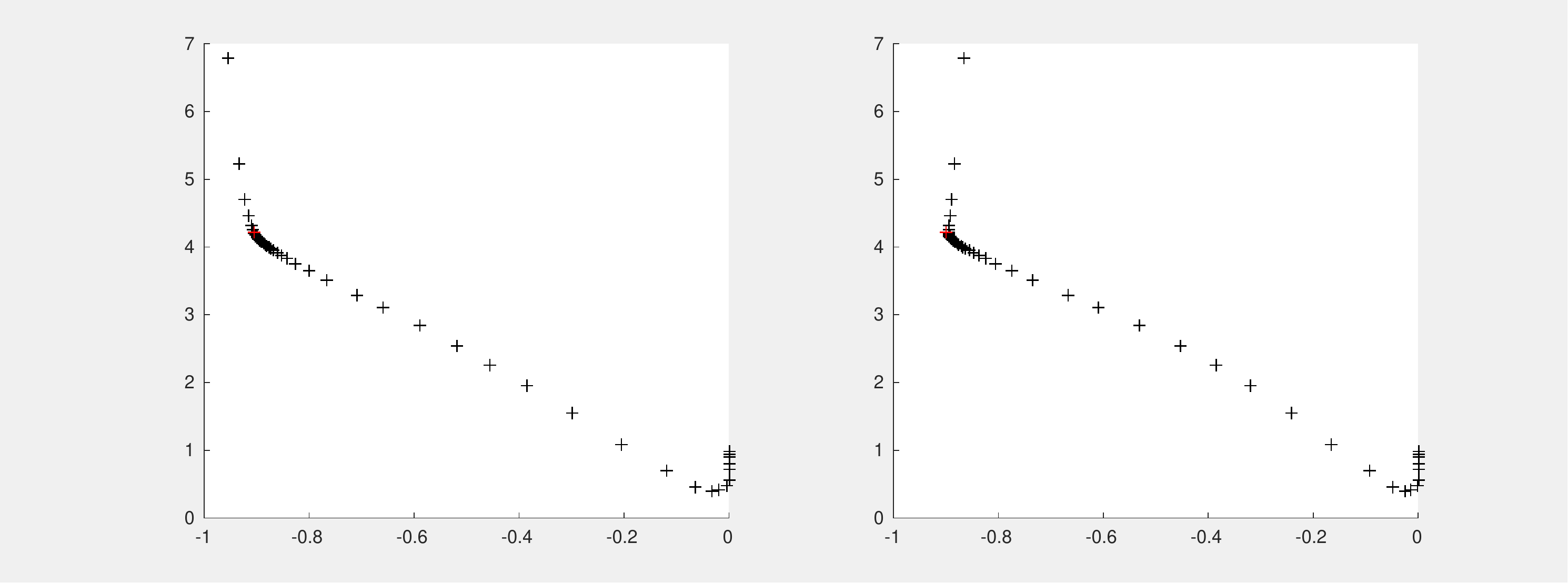}
\caption{Pareto curves of our estimators. Left panel is the plot of of the estimated correlation $\hat u ^TX^TY \hat v$ versus the sum of $\|\hat u\|_{\ell_1}$ and $\|\hat v\|_{\ell_1}$, averaged over 100 simulations. The red dot corresponds to the ($u^TX^TY v$, $\|u\|_{\ell_1} + \|v\|_{\ell_1}$). Right panel is the plot of the estimated correlation $\hat u ^T\Sigma_{XY}\hat v$ versus the sum of $\|\hat u\|_{\ell_1}$ and $\|\hat v\|_{\ell_1}$, averaged over 100 simulations. The red dot corresponds to the ($u^T\Sigma_{XY} v$, $\|u\|_{\ell_1} + \|v\|_{\ell_1}$). Note that the red dot is on the pareto curve, which means that our algorithm achieve this optimal value with right choice of regularizers. } 
\label{fig:example0}
\end{figure}

\begin{figure}
\includegraphics[scale= 0.35]{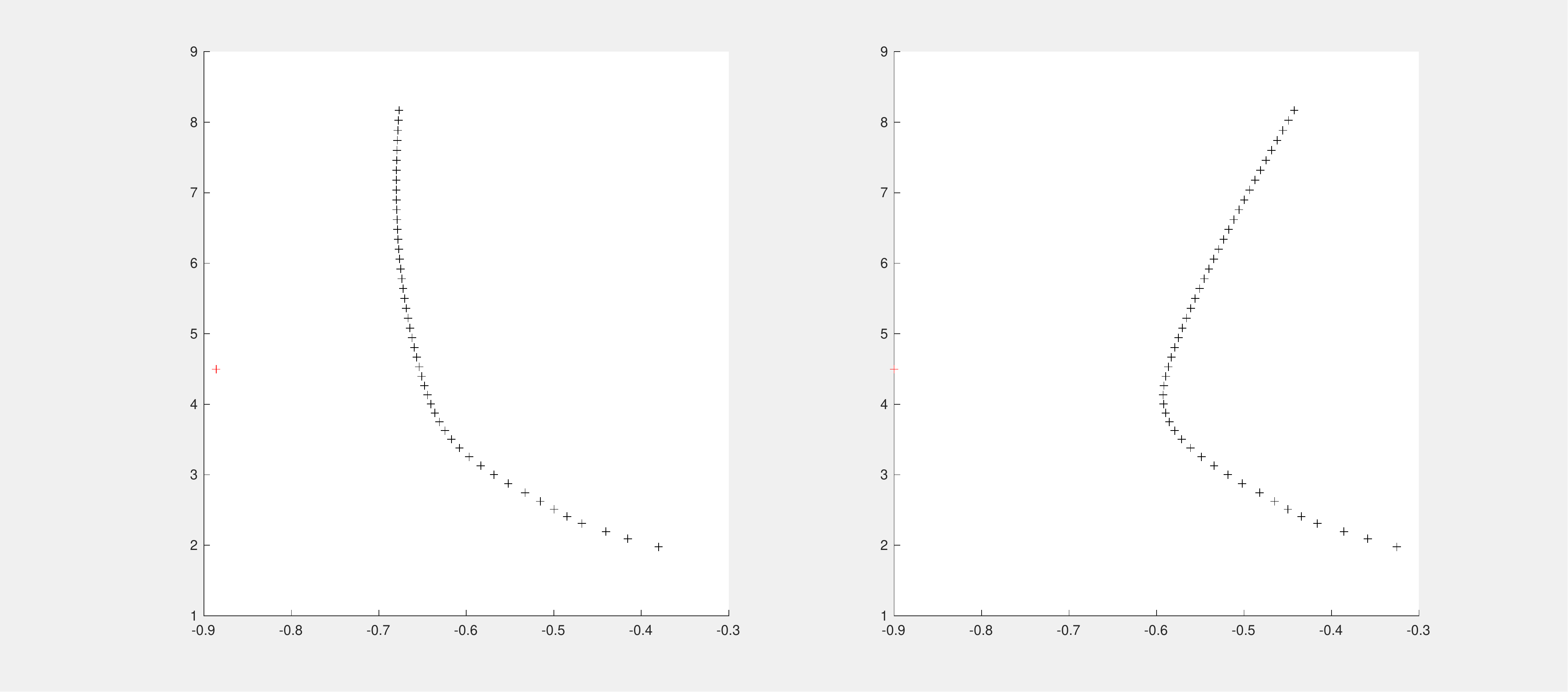}
\caption{Pareto curves of \citet{Witten}. Left panel is the plot of of the estimated correlation $\hat u ^TX^TY \hat v$ versus the sum of $\|\hat u\|_{\ell_1}$ and $\|\hat v\|_{\ell_1}$, averaged over 100 simulations. The red dot corresponds to the ($u^TX^TY v$, $\|u\|_{\ell_1} + \|v\|_{\ell_1}$). Right panel is the plot of the estimated correlation $\hat u ^T\Sigma_{XY}\hat v$ versus the sum of $\|\hat u\|_{\ell_1}$ and $\|\hat v\|_{\ell_1}$, averaged over 100 simulations. The red dot corresponds to the ($u^T\Sigma_{XY} v$, $\|u\|_{\ell_1} + \|v\|_{\ell_1}$). Note that the red dot is not on the pareto curve, which means that their algorithm could not achieve this optimal value with any choice of regularizers. } 
\label{fig:example0witten}
\end{figure}

The left panel of \autoref{fig:example0} is the plot of of the estimated correlation $\hat u ^TX^TY \hat v$ versus the sum of $\|\hat u\|_{\ell_1}$ and $\|\hat v\|_{\ell_1}$, averaged over 100 simulations. The right panel is the plot of  estimated correlation $\hat u ^T\Sigma_{XY}\hat v$ versus the sum of $\|\hat u\|_{\ell_1}$ and $\|\hat v\|_{\ell_1}$, averaged over 100 simulations. Note that
 we replace the sample covariance with the true covariance. From both panels, with the right choice of regularizers, our algorithm can achieve the optimal values. However, as shown in \autoref{fig:example0witten}, the solutions of \citet{Witten} are  very far from the true solution. The red dots are not on their solutions' path, meaning that their results do not achieve the optimal value with any choices of regularizers.
\end{example}
\subsection{Algorithmic details}
\label{algorithmic details}
\subsubsection{Linearized alternating direction minimization method}
We assume that the data matrix $X$ and $Y$ are centred. We now present how to solve the minimization problem \autoref{solveforu} in detail, and the algorithm works similarly for $v$. 

With the data matrix $X$ and $Y$, the minimization \autoref{solveforu} becomes
\begin{align}
\underset{u}{\text{minimize}} -u^TX^TYv + \tau_1\|u\|_1 + \Ind\{u:
 \|Xu\|_2 \leq 1\}. 
\label{dataversion: solveforu}
\end{align}
Let $z = Xu$, we have
\begin{align}
\underset{u}{\text{minimize}} \ \ \ &\underbrace{-u^TX^TYv + \tau_1\|u\|_1}_{f(u)} + \underbrace{\Ind\{ \|z\|_2 \leq 1\}}_{g(z)}\nonumber\\
\text{subject to } \ \ \ &Xu = z
\label{foru}
\end{align}
We can use linearized alternating direction method of multipliers \citep*{ADMM} to solve this problem. The alternating direction method of multipliers is to solve the augmented Lagrangian by solving each variable and the dual variable one by one until convergence. The detailed derivation can be seen in Appendix and the complete algorithm can be seen in Algorithm \ref{alg:CCA}.

\subsubsection{TFOCS}
The other approach to solve \autoref{solveforu} is to use TFOCS. We rewrite the \autoref{solveforu} as follows and use \texttt{tfocs\_SCD} function to solve it. 

Since $v$ is fixed, and let $c = v^TY^TX $, minimizing the objective function of \autoref{dataversion: solveforu} 
\begin{align}
&-cu + \tau_1\|u\|_1 + \Ind\{\|Xu\|_2 \leq 1\}
\label{solveforuwithc}
\end{align}
is equivalent to minimizing 
\begin{align*}
&-cu/\tau_1 + \|u\|_1 + \Ind\{\|Xu\|_2 \leq 1\}
\end{align*}
Instead of solving this objective function, we solve instead
\begin{align*}
\|u\|_1 + \frac{1}{2}\mu\|u - (u_{old} + \frac{1}{\tau \mu}c)\|_2^2+ \Ind\{\|Xu\|_2 \leq 1\}
\end{align*}
Intuitively, we solve the \autoref{solveforuwithc} without going too far from current approximation. This formulation can be solved using \texttt{tfocs\_SCD}. \cite{TFOCS}. 

\begin{algorithm}
   \caption{Sparse CCA}
\label{alg:CCA}
    \begin{algorithmic}[1]
      \Function{CCA}{$X, Y$}
        \State Initialize $u_0$ and $v_0$
	\While {not converged}
	\State Fix $v_k$
		\While{not converged}
			\begin{align*}
				u^{k+1} &\leftarrow \prox_{\mu f}(u^k - \frac{\mu}{\lambda}X^T(Xu^k - z^k + \xi^k))\\
				z^{k+1} &\leftarrow \prox_{\lambda g}(Xu^{k+1} + \xi^k)\\
				\xi^{k+1} &\leftarrow \xi^k  + Xu^{k+1} - z^{k+1}
			\end{align*}
		\EndWhile
	\State Fix $u_k$,
		\While{not converged}
			\begin{align*}
				v^{k+1} &\leftarrow \prox_{\mu f}(v^k - \frac{\mu}{\lambda}Y^T(Yv^k - z^k + \xi^k))\\
				z^{k+1} &\leftarrow \prox_{\lambda g}(Yv^{k+1} + \xi^k)\\
				\xi^{k+1} &\leftarrow \xi^k  + Yv^{k+1} - z^{k+1}
			\end{align*}
	\EndWhile
	\EndWhile
            \EndFunction
\end{algorithmic}
\end{algorithm}

%

\subsection{The remaining canonical vectors}
\label{sub: remaining canonical vectors}
Given the first $r-1$ canonical vectors $U = \begin{pmatrix}u_1& \cdots& u_{r-1}\end{pmatrix}$ and $\begin{pmatrix}v_1& \cdots& v_{r-1}\end{pmatrix}$, we consider solving the $r$-th canonical vectors by 
\begin{align*}
\underset{u,v}{\text{minimize}} \ \ \ & -u^TX^TYv + \tau_1|u|_{1}  + \tau_2|v|_1 + \Ind\{u: \|Xu\|_2 \leq 1\} + \Ind\{v:\|Yv\|_2 \leq 1\} \\
\text{subject to } \ \ \ & U^TX^TXu = 0; V^T Y^TYv = 0. 
\end{align*}
This problem is biconvex, and we use the same approach of fixing one variable and solve for the other one. Fixing $v$, we get $\hat u$ by solving
\begin{align*}
\underset{u}{\text{minimize}} \ \ \ & -u^TX^TYv + \tau_1|u|_{1} + \Ind\{z: \|z\|_2 \leq 1\}\\
\text{subject to } \ \ \ & Xu = z; U^TX^TXz = 0_{r-1}, 
\end{align*} 
and fixing $u$, we get $\hat v$ by solving
\begin{align*}
\underset{v}{\text{minimize}} \ \ \ & -u^TX^TYv + \tau_1|v|_{1} + \Ind\{z: \|z\|_2 \leq 1\}\\
\text{subject to } \ \ \ & Yv = z; V^TY^TYz = 0_{r-1}.  
\end{align*} 
The constraint can be combined as 
\begin{align*}
\begin{pmatrix}
X\\
U^TX^TX
\end{pmatrix}u -\begin{pmatrix}I\\0\end{pmatrix}z = 0 \quad 
\begin{pmatrix}
Y\\
V^TT^TY
\end{pmatrix}v -\begin{pmatrix}I\\0\end{pmatrix}z = 0
\end{align*}
Let $\tilde{X} = \begin{pmatrix}
X\\
U^TX^TX
\end{pmatrix}$, fixing $v$, we can easily see that 
\begin{align*}
-u^T X^TYv =-u^T \tilde{X}^TYv, 
\end{align*}
and fixing $u$,
\begin{align*}
-u^T X^T Y v = -u^T X^T \tilde{Y}v. 
\end{align*}
Therefore, we can use the linearized ADMM with the new matrix $\tilde{X}$ and $\tilde{Y}$ to get the $r$-th canonical vectors.   

\subsection{A bridge for the covariance matrix}
As mentioned in \autoref{sec: main sec}, \cite{Witten} proposed to replace the covariance matrix with an identity matrix. Since their solution can be solved efficient, it is of interest to investigate the relation between our method and theirs. Therefore, we now write the covariance matrix as 
\begin{align*}
\alpha_x X^TX  + (1-\alpha_x) \identity_{p_u, p_u}. 
\end{align*}
We can replace similarly for $Y$. 

The constraint $\|Xu\|_2^2$ gives 
\begin{align*}
u^T(\alpha_x X^TX  + (1-\alpha_x) \identity)u &= \alpha_x \|Xu\|_2^2 + (1-\alpha_x)\|u\|_2^2 =\| \begin{pmatrix}\alpha_x X \\ (1-\alpha) \identity_{p_u, p_u} \end{pmatrix}u\|_2^2. 
\end{align*}
This form can be solved using the methods we proposed by changing the linear operator with the matrix above. If interested to see how solutions change from \citet{Witten} to our solution, one can use the above to see the path using different choices of $\alpha_x, \alpha_y$.

\subsection{Semidefinite Programming Approach}
We now show that \autoref{biconvex} can be solved using a semi-definite programming approach. This idea is not new, but borrowed from the approach to solve sparse principle components \citep{sparsePCA} with some modifications.  Let $h = \begin{pmatrix}
u \\ v \end{pmatrix}$, the problem of 
\begin{equation}
\begin{array}{ll}
    \mbox{minimize}   & -u^TX^TYv\\
    \mbox{subject to} & u^TX^TXu = 1 ;  v^TY^TYv = 1 \\
    & |v| \leq t_v ; |u|\leq t_u
    \end{array}
    \label{global}
\end{equation}
can be written as
\begin{equation}
\begin{array}{ll}
    \mbox{minimize}   & -\frac{1}{2}h^T \begin{pmatrix}0&X^TY\\Y^TX& 0\end{pmatrix} h   \\
    \mbox{subject to} & h^T\begin{pmatrix}
    X^TX & 0 \\ 0 & 0\end{pmatrix} h  =1;  h^T\begin{pmatrix}
    0 & 0 \\ 0 & Y^TY\end{pmatrix}  h =1   \\
   & |h 1_v|  \leq t_v ; |h1_u|\leq t_u
    \end{array}
    \label{beforetrace}
\end{equation}
Now, we transfer the objective function using a trace operation: 
\begin{align*}
-\frac{1}{2}h^T \begin{pmatrix}0&X^TY\\Y^TX& 0\end{pmatrix} h  = -\trace(\begin{pmatrix}0&X^TY\\Y^TX& 0\end{pmatrix} hh^T)\
\end{align*}
Let $H = hh^T$ and 
\[
	Q = \begin{pmatrix}0&X^TY\\Y^TX& 0\end{pmatrix}; Q_X = \begin{pmatrix}
    X^TX & 0 \\ 0 & 0\end{pmatrix}; Q_Y = \begin{pmatrix}
    0 & 0 \\ 0 & Y^TY\end{pmatrix}
\]

\begin{align*}
h^T\begin{pmatrix}X^TX & 0 \\ 0 & 0\end{pmatrix} h = \trace(\begin{pmatrix}X^TX & 0 \\ 0 & 0\end{pmatrix}H)
\end{align*}

\begin{equation}
\begin{array}{ll}
    \mbox{minimize}   &  -\trace(QH) + \lambda\sum_{i,j}|H_{ij}|\\
    \mbox{subject to} & H\in \semi_{+}^{p_u + p_v}\\
    &\trace(Q_XH) =1; \trace(Q_yH) =1 
    \end{array}
    \label{aftertrace}
\end{equation}
Semi-definite programming problem can be very computational expansive, especially when $p$ is much greater than $n$ Therefore, we do not compute the sparse canonical vectors using this formulation. It would be a interesting direction to explore if there exists an efficient algorithm to solve this problem efficiently. 
\section{A Special Case}
\label{sec: a special case}
In this section, we consider a special case, where the covariance matrices of $x$ and $y$ is identity. 
Suppose that the matrices $\Sigma_x = \Sigma_y = I$, and thus the covariance matrix $\Sigma_{xy} = U\Lambda V^T$, where $U\in \Real^{p \times k}, V\in \Real^{q \times k}$, and $\Lambda\in \Real^{k\times k}$ is diagonal.  In other words, $\Sigma_{xy}$ is rank $k$.  We now show that our problem is similar to solving a sparse principle component analysis. Note that $U^TU = I_k$ and $V^TV = I_k$.
\begin{align*}
\begin{pmatrix}
X \\ Y  \end{pmatrix} \sim \normal\big( \begin{pmatrix}
0 \\ 0  \end{pmatrix}  \begin{pmatrix} I & U \Lambda V^T \\ V \Lambda U^T & I \end{pmatrix} \big)
\end{align*}
\begin{theorem}
\label{theorem: sparse pca}
Estimation of $u$ and $v$ can be obtained using spectral decomposition and thus we can use software which solves sparse principle components to solve the problem above. 
\end{theorem}
\begin{proof}
Let $\Sigma =  A + I = \begin{pmatrix} 0 & U \Lambda V^T \\ V \Lambda U^T & 0 \end{pmatrix} + \begin{pmatrix}
I & 0\\0 & I
\end{pmatrix} $.  
\begin{align*}
A = \begin{pmatrix} 0 & U \Lambda V^T \\ V \Lambda U^T & 0 \end{pmatrix} = \frac{1 }{2}\begin{pmatrix}
U\\V
\end{pmatrix} \Lambda \begin{pmatrix}
U^T & V^T
\end{pmatrix}  - \frac{1}{2} \begin{pmatrix}
U \\ -V 
\end{pmatrix}\Lambda \begin{pmatrix}
U^T & -V^T 
\end{pmatrix}
\end{align*}
Let $U_i, V_i$ denote the $i$th columns of $U$ and $V$ respectively, and denote 
$$W_i = \frac{1}{\sqrt{2}} \begin{pmatrix} U_i \\ V_i \end{pmatrix}, \qquad W_{k+i} = \frac{1}{\sqrt{2}} \begin{pmatrix} U_i \\ -V_i \end{pmatrix}$$
for $i=1,\dots,k$.  Note that $W_i^T W_j = \mathbb{I}(i=j)$, for $i,j = 1,\dots,2k$.  Let $\{W_i\}_{i=2k+1}^{p_u + p_v}$ be an orthonormal set of vectors orthogonal to $\{W_i\}_{i=1}^{2k}$.  Then the matrix $\Sigma = A+I$ has the following spectral decomposition
$$\Sigma = \sum_{i=1}^k (1+\Lambda_{i,i}) W_i W_i^T + \sum_{i=k+1}^{2k}(1-\Lambda_{i-k,i-k}) W_i W_i^T + \sum_{i=2k+1}^{p_u+p_v} W_i W_i$$

Therefore, $\Sigma$ can be thought as a spiked covariance matrix, where the signal to noise ratio (SNR) can be interpreted as $1+\min_i \Lambda_{i,i}$. We know that in the high dimensional regime, if 
\begin{align*}
\text{SNR} \geq \sqrt \frac{p}{n}
\end{align*}
we can recover $u$ and $v$ even if $u$ and $v$ are not sparse. However, if 
\begin{align*}
\text{SNR} < \sqrt {\frac{p}{n}}, 
\end{align*} 
we need to enforce the sparsity in $u$ and $v$, see \citet{Bphase} and \citet{paul2007} for details. 
\end{proof}

We can see from \autoref{theorem: sparse pca} that if the covariance matrices of $x$ and $y$ are identity, or act more or less like identity matrices, solving canonical vectors can be roughly viewed as solving sparse principle components. Therefore, in this case, estimating canonical vectors is roughly as hard as solving sparse eigenvectors. 

\section{Simulated Data}
\label{sec: simulated data}

In this section, we carefully analyze different cases of covariance structure of $x$ and $y$ and compare the performance of our methods with other methods. We first explain how we generate the data. 

Let $X \in \Real^{n \times p}$ and $Y \in \Real^{n \times q}$ be the data generated from the model 
\begin{align*}
\begin{pmatrix} x \\ y \end{pmatrix} \sim \normal \big( \begin{pmatrix} 0\\0 \end{pmatrix}, \begin{pmatrix}\Sigma_x & \Sigma_{xy}\\ \Sigma_{yx} & \Sigma_y\end{pmatrix}\big ),
\end{align*}
where $\Sigma_{xy} = \rho \Sigma_x u v^T \Sigma_y$, where $u$ and $v$ are the true canonical vectors, and $\rho$ is the true canonical correlation. We would like to estimate $u$ and $v$ from the data matrices $X$ and $Y$. We compare our methods with other methods available on different choices of triplets: $(n, p, q)$, where $n$ is the number of samples, $p$ is the number of features in $X$, and $q$ is the number of features in $Y$. 
In order to measure the discrepancy of estimated $\hat u$, $\hat v$ with the true $u$ and $v$, we use  the sin of the angle between $\hat u$ and $u$, $\hat v$ and $v$ \citeauthor{johnstone} 
\begin{align}
\loss(\hat v, v) &= \min(\|\hat v - v\|_2^2, \|\hat v+ v\|_2^2\\
&= 2(1- |\left\langle \hat v, v \right\rangle|),
\end{align}
where $\|\hat v\|_2 = \|v\|_2 = 1$. 
\subsection{Identity-like covariance models}
In sparse canonical correlation analysis literature, structured covariance of $x$ and $y$ are highly investigated. For examples, covariance of $x$ may be identity covariance, toeplitz, or have sparse inverse covariance.  From the plot of the covariances matrix in \autoref{fig:topelitz} and \autoref{fig:sparseinverse}, we can see toeplitz and sparse inverse covariance act more or less like identity matrices. Since the covariance of $x$ and $y$ act more or less like identity matrices, as discussed in previous section, solving $u$ and $v$ is roughly as hard as solving sparse eigenvectors. In other words, the covariance of $x$ and $y$ do not change the signal in $u$ and $v$ much and as a result, the signal in $\Sigma_xy$ is very sparse.  In this case, an initial guess is very important. We propose the following procedure:
\begin{enumerate}
\item Denoise the matrix $X^TY$ by solft-thresholding the matrix elementwise, call the resulting matrix as $S_{xy}$. 
\item We obtain the initial guess as follows:
	\begin{enumerate}
		\item Take singular value decomposition of $S_{xy}$, denoted as $\hat U$ and $\hat V$. 
		\item Normalize the each column $u_i$, $v_i$ in $\hat U$ and $\hat V$ by 
		$u_i \leftarrow \frac{u_i}{\sqrt{u_i^T (X^T  X)  u_i}}$	and $v_i \leftarrow \frac{v_i}{\sqrt{v_i^TY^T  Y v_i}}$. Denote the resulting $\hat U$ and $\hat V$ as $\tilde{U}$ and $\tilde{V}$. 
		\item Calculate $\tilde D = \tilde{U}^T X^T Y \tilde{V}$ Choose the index $k$ where the maximum diagonal element of $\tilde D $ is obtained, i.e., $\diag(D)_k = \max\{\diag(D)\}$
	\end{enumerate}
\item Use the initial guess to start the alternating minimization algorithm.
\end{enumerate}

We consider three types of covariance matrices in this category: toeplitz, identity, and sparse inverse matrices. 
\begin{figure}
\includegraphics[scale=0.5]{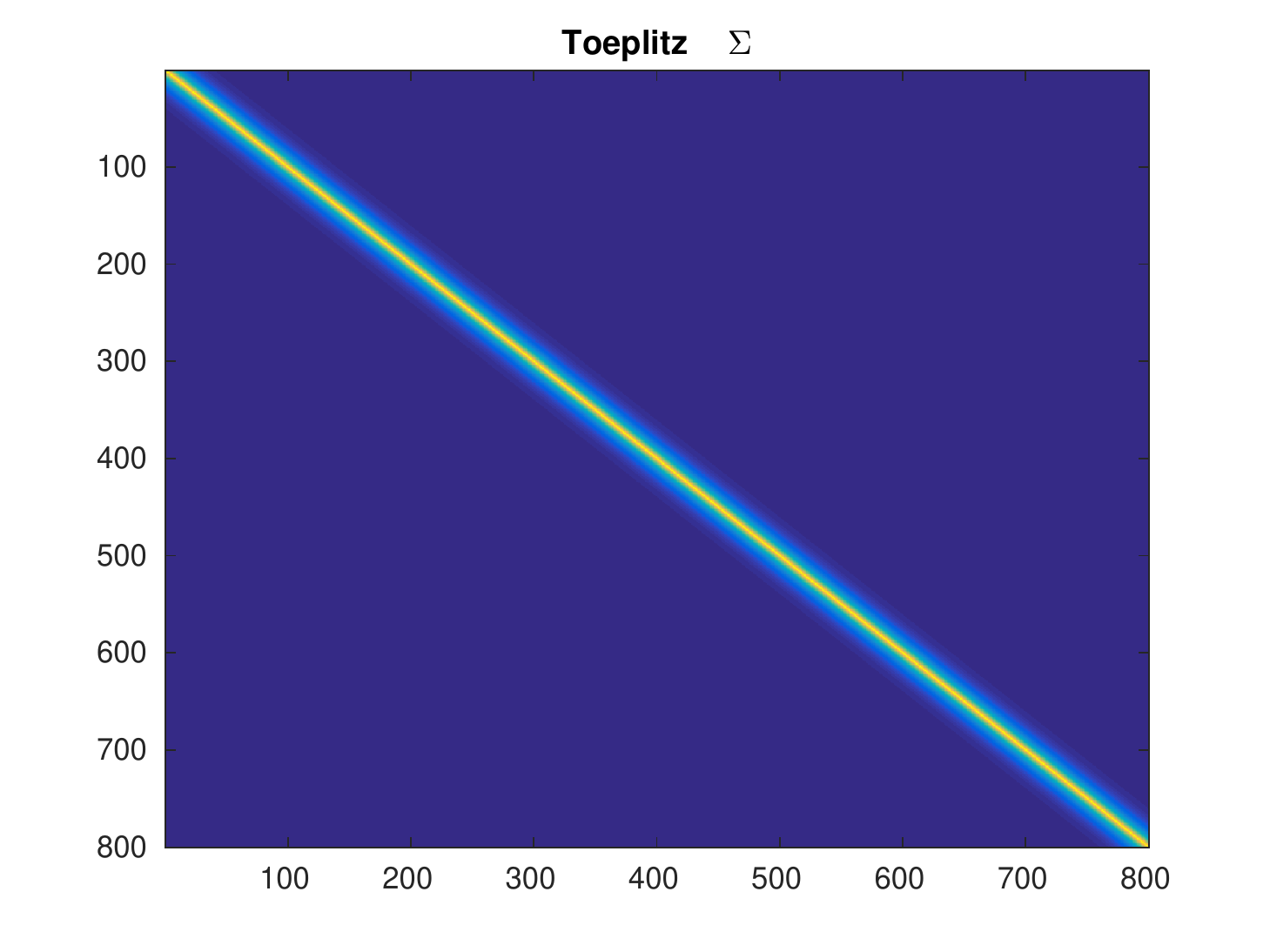}
\includegraphics[scale=0.5]{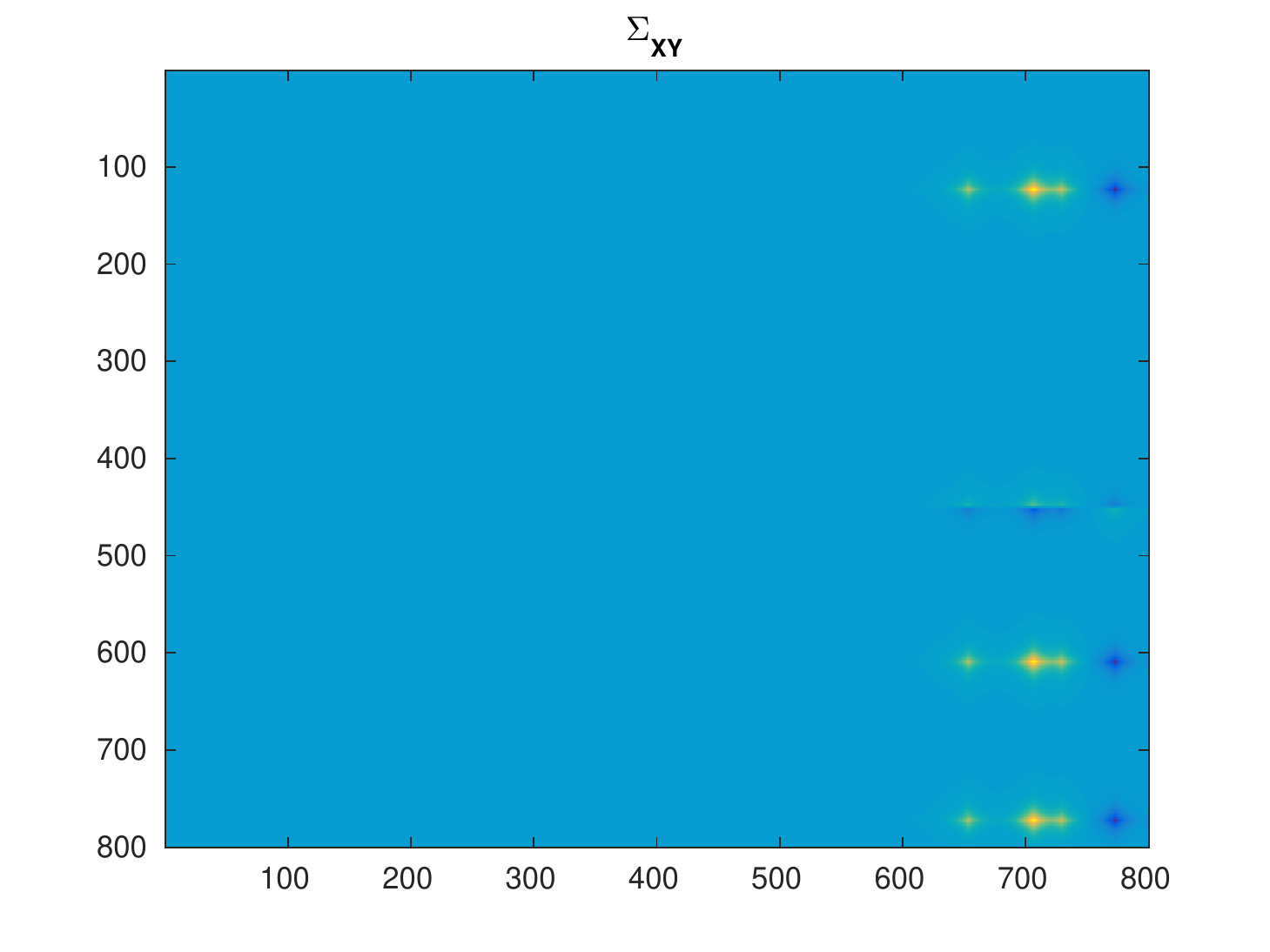}
\caption{\textit{Toeplitz matrices with $\sigma_{ij} = 0.9^{|i-j|}$: we can see that even though it is not exactly identity matrix, the general structure does look like identity matrix. }}
\label{fig:topelitz}
\end{figure}

\begin{figure}
\includegraphics[scale=0.5]{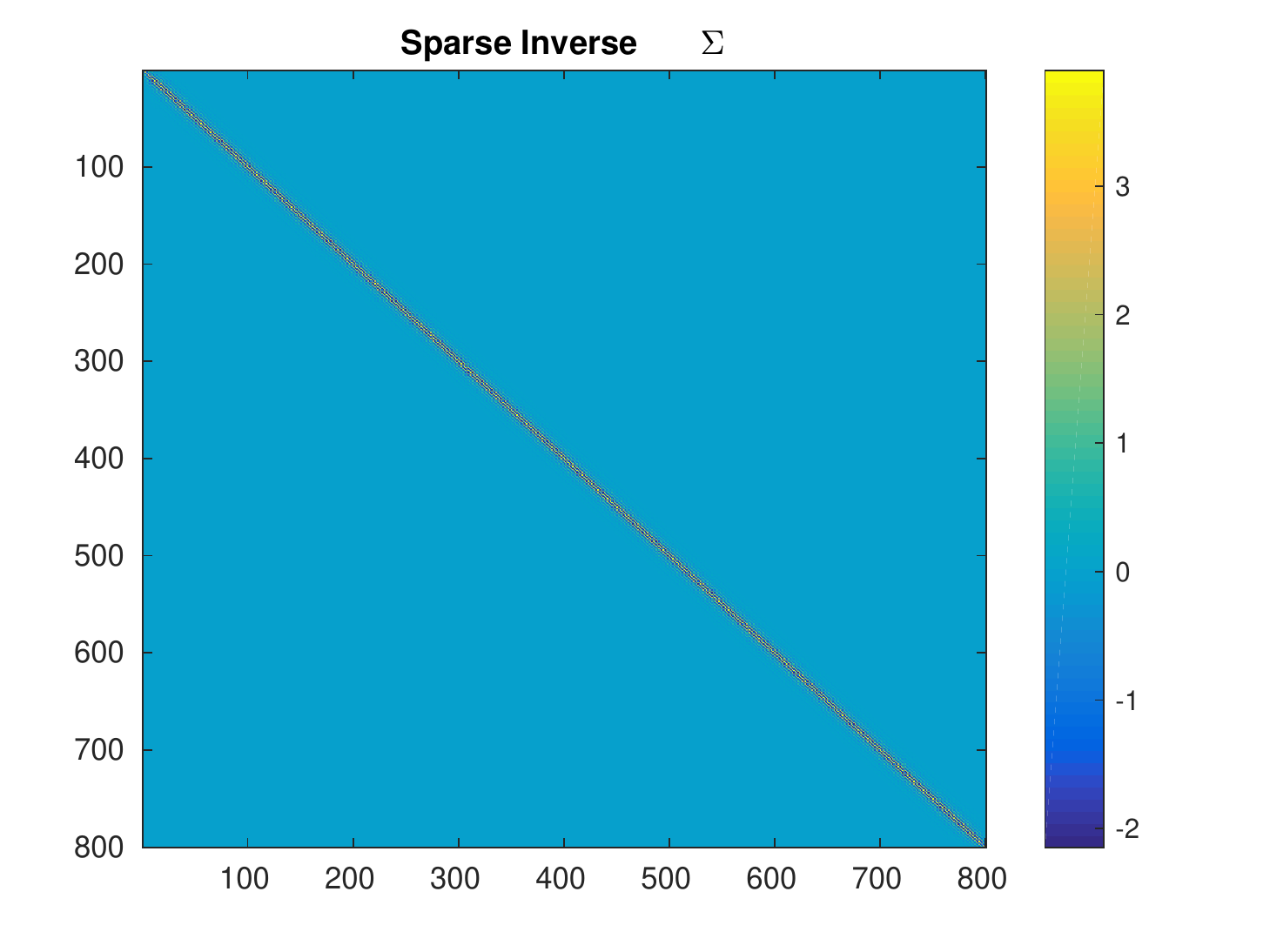}
\includegraphics[scale=0.5]{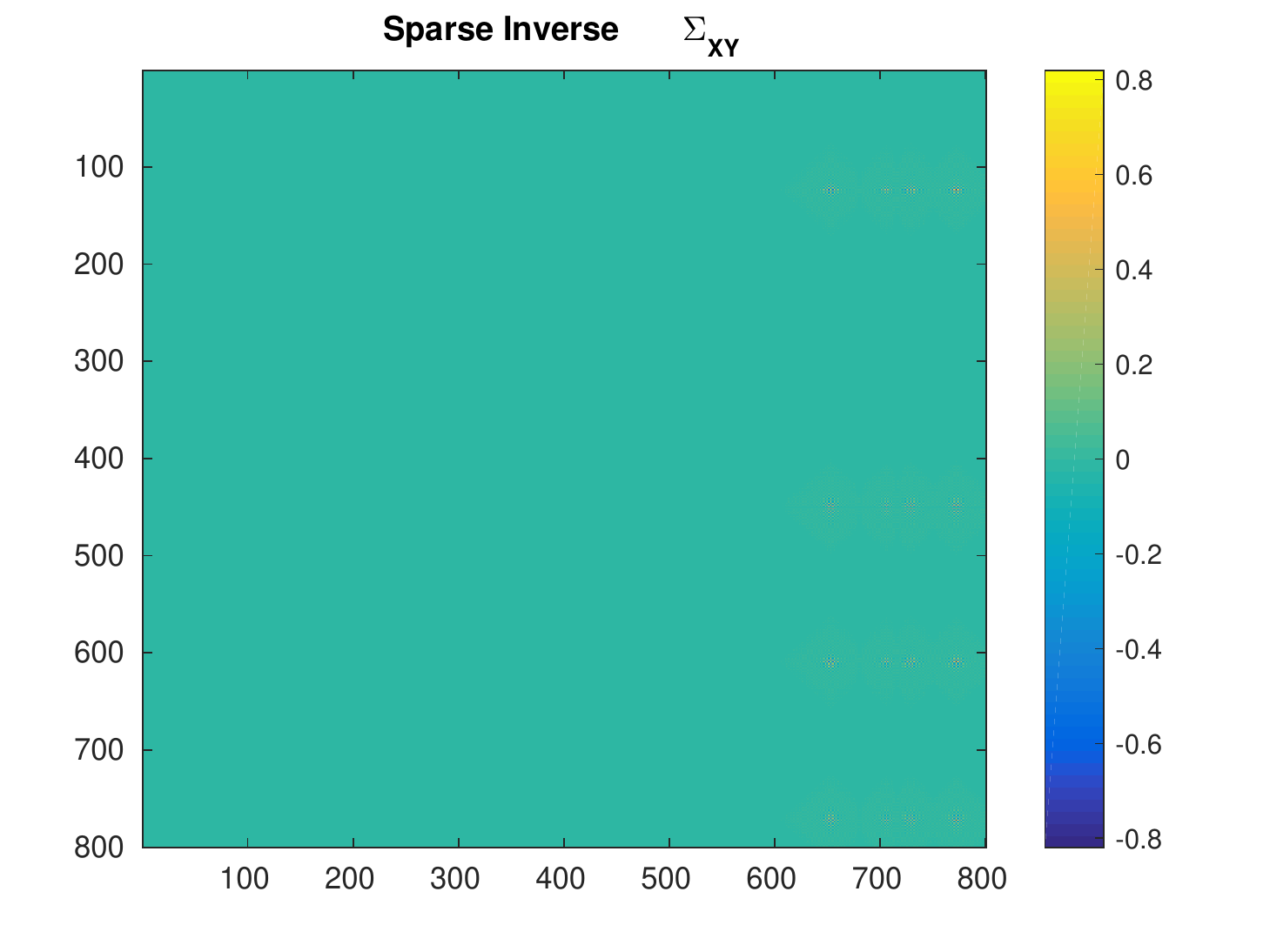}
\caption{\textit{Sparse inverse matrix. }}\label{fig:sparseinverse}
\end{figure}

\begin{enumerate}
\item $\Sigma = I_p$.  
\item $\Sigma = (\sigma_{ij})$, where $\sigma_{ij} = 0.9^{|i-j|}$ for all $i,j \in p, q$. Here $\Sigma$ are Toeplitz matrices. See the plot of the toeplitz matrx and its corresponding $\Sigma_{xy} = \Sigma_x \rho uv^T \Sigma_y$. We can see that though it is not identity matrix, it behaves more or less like an identity matrix. Note that the smaller the toeplitz constant is, the more it looks like an identity matrix.

\item $\Sigma = (\frac{\sigma^0_{ij}}{\sqrt{\sigma^0_{ii}}\sigma^0_{jj}})$. Let $\Sigma^0 = (\sigma_{ij}^0) = \Omega^{-1}$ where $\Omega = (\omega_{ij})$ with 
\begin{align*}
\omega_{ij} = \Ind_{\{i = j\}} + 0.5 \times \Ind_{\{|i-j| = 1\}}  + 0.4 \times \Ind_{\{|i-j| = 2\}}, i, j \in [p]
\end{align*}
In this case, $\Sigma_x$ and $\Sigma_y$ have sparse inverse matrices. 
\end{enumerate} 

In each example, we simulate 100 data sets, i.e., 100 $X$, and 100 $Y$ in order to average our performance. We set the number of non-zeros in the $u$ and $v$ to be 5, the index of nonzeros are randomly chosen. We will vary the number of nonzeros in the next comparison. For each simulation, we have a sequence of regularizer $\tau_u$ and $\tau_v$ to choose from. For simplicity, we chose the best $\tau_u$ and $\tau_v$ such that estimated $\hat u$ and $\hat v$ minimize the $\loss$ defined above in every methods. 

We present our result in the \autoref{tab: identity},  \autoref{tab: toeplitz} and \autoref{tab: sparseinv}.  There are some notations presented in the table and we now explain them here. $\hat \rho$ is the estimated canonical correlation between data $X$ and $Y$. $e_u = \loss(\hat u , u)$ and $e_v = \loss(\hat u , u)$. We compare our result with the methods proposed by \citet{Witten}, and \citet{sparseCCA}. Since we are not able to run the code from \citet{sparseig} very efficiently, we will compare our method with their approach in the next subsection.  In order to compare them in the same unit, we calculate the estimates of each method and then normalize them by $X \hat u$, and $Y\hat v$ respectively.  We then normalize estimates such that they all have norm 1. We report the estimated correlation, loss of $u$ and loss of $v$ as a format of $(\rho, e_u, e_v)$ for each method in all tables. 
From \autoref{tab: identity}, \autoref{tab: toeplitz}, and \autoref{tab: sparseinv},  we can see that SCCA method proposed by \citet{sparseCCA} performs similarly with ours. However, their two step procedure is computationally expensive compared to ours and hard to choose regularizers. Estimates by \citet{Witten} fails to provide accurate approximations because of the low samples size we considered.  

\begin{table}
\centering
\begin{tabular}{c|c|c|c}    \toprule
\emph{$(n, p, q)$} & Our method&SCCA& PMA \\\midrule
(400, 800,800) &(0.90,0.056,0.062)&(0.90,0.060,0.066) &  (0.71,1.17, 1.17)\\
(500, 600, 600) & (0.90,0.05,0.056)&(0.90,  0.053, 0.057)   & (0.71,0.85,0.85)\\
(700, 1200,1200) &(0.90,0.045, 0.043)&(0.90,0.045, 0.043)&(0.71, 1.09,1.09)\\\bottomrule
\hline
\end{tabular}
\caption{\textit{Error comparison for identity matrices: we use a format of $(\hat \rho, \loss(\hat u), \loss(\hat v)$ to represent each method's result. }}
\label{tab: identity}
\end{table}

\begin{table}
\centering
\begin{tabular}{c|c|c|c}    \toprule
\emph{$(n, p, q)$} & Our method&SCCA& PMA \\\midrule
(400, 800, 800) &  (0.91, 0.173 ,0.218)& (0.91, 0.213, 0.296) & (0.52,1.038,1.067)\\
(500, 600, 600) &(0.90, 0.136, 0.098)&(0.90, 0.145, 0.109) &(0.55, 1.11, 0.94)\\
(700, 1200, 1200) & (0.90, 0.109, 0.086) &(0.90, 0.110, 0.088)  &(0.60, 1.098,0.89) \\
\bottomrule
\hline
\end{tabular}
\caption{Error comparison for toeplitz matrices: \textit{we use a format of $(\hat \rho, \loss(\hat u), \loss(\hat v)$ to represent each method's result.  }}
\label{tab: toeplitz}
\end{table}

\begin{table}
\centering
\begin{tabular}{c|c|c|c}    \toprule
\emph{$(n, p, q)$} & Our method&SCCA& PMA \\\midrule
(400, 800, 800) & (0.92,0.092,0.149)& (0.92,0.129, 0.190) & (0.61, 0.93, 1.0)  \\
(500, 600, 600) & (0.90, 0.068, 0.059)&  (0.90, 0.069, 0.0623) & (0.7215, 0.67   0.45)\\
(700, 1200, 1200) &(0.90, 0.050 ,0.044)&  (0.90, 0.051, 0.047)&  (0.70, 0.76,     0.58)\\\bottomrule
\hline
\end{tabular}
\caption{\textit{Error comparison for sparse inverse matrices:  we use a format of $(\hat \rho, \loss(\hat u), \loss(\hat v)$ to represent each method's result. }}
\label{tab: sparseinv}
\end{table}

\subsection{Spiked covariance models}
In this subsection, we consider covariance matrices of $x \in \reals^p$ and $y \in \reals^q $ are spiked, i.e., 
\begin{align*}
\cov(x) =\sum_{i = 1}^{k_x} \lambda_iw_i w_i^T + I_{p},  \\
\cov(y) = \sum_{i = 1}^{k_y} \lambda_iw_i w_i^T + I_{q}.
\end{align*}

In Example \autoref{ex:svd} we will see that even we have the more observations with the number of features,  the traditional singular value decomposition  can return bad results. 

\begin{example}
\label{ex:svd}
We generate the $\Sigma_x$ and $\Sigma_y$ as follows: 
\begin{align*}
\Sigma_{x} &= \sum_{i = 1}^k \lambda_{x,i} w_{x,i} w_{x,i}^T + I\\
\Sigma_{y} &= \sum_{i = 1}^k \lambda_{y,i} w_{y,i} w_{y,i}^T + I
\end{align*}
where $w_{x,1}, \cdots w_{x, k}$,  $\reals^{p}$, $w_{y,1}, \cdots w_{y, k}$ are independent orthonormal vectors in $\reals^{p}$,  $\reals^q$ respsectively, and $\lambda_{x,i} = \lambda_{y, i} = 250$ and $k = 20$.  
The covariance $\Sigma_{xy}$ is generated as 
\begin{align*}
\Sigma_{xy} = \Sigma_x \rho u v^T \Sigma_y, 
\end{align*}
where $u$ and $v$ are the true canonical vectors and have 10 nonzero elements with indices randomly chosen. We generate the data matrices $X\in \reals^{n \times p}$ and $Y^{n\times q}$ from the distribution 
\begin{align*}
\begin{pmatrix}
x\\y
\end{pmatrix}\sim \normal (\begin{pmatrix}
0\\0
\end{pmatrix}, \begin{pmatrix}
\Sigma_x & \Sigma_{xy} \\
\Sigma_{xy}^T & \Sigma_y
\end{pmatrix}. 
\end{align*}
Therefore, when $n = 1000, p = 800, q = 800$, we should be able to estimate $u$ and $v$ using the singular value decomposition of the matrix 
\begin{align*}
\hat\Sigma_x^{-1/2}\hat \hat{\Sigma}_{xy} \hat{\Sigma}_y ^{-1/2} = (X^T X)^{-1/2} (X^T Y) (Y^T Y)^{-1/2}. 
\end{align*}
However, the estimated $\hat u$ and $\hat v$ can be seen in \autoref{svduv}. The results are wrong and not sparse. This is an indication that we need more samples to estimate the canonical vectors. As we increase the sample size to $n =3000$, estimates of $u$ and $v$ are more accurate but not very sparse, as seen in  \autoref{svduvmore}. For our method, we use $n = 400$,  the estimated $\hat u$ and $\hat v$ of our methods can be seen in \autoref{ouruv}. Our method returns sparse and better estimates for $u$ and $v$. 

\begin{figure}
\centering
\includegraphics[height = 6 cm]{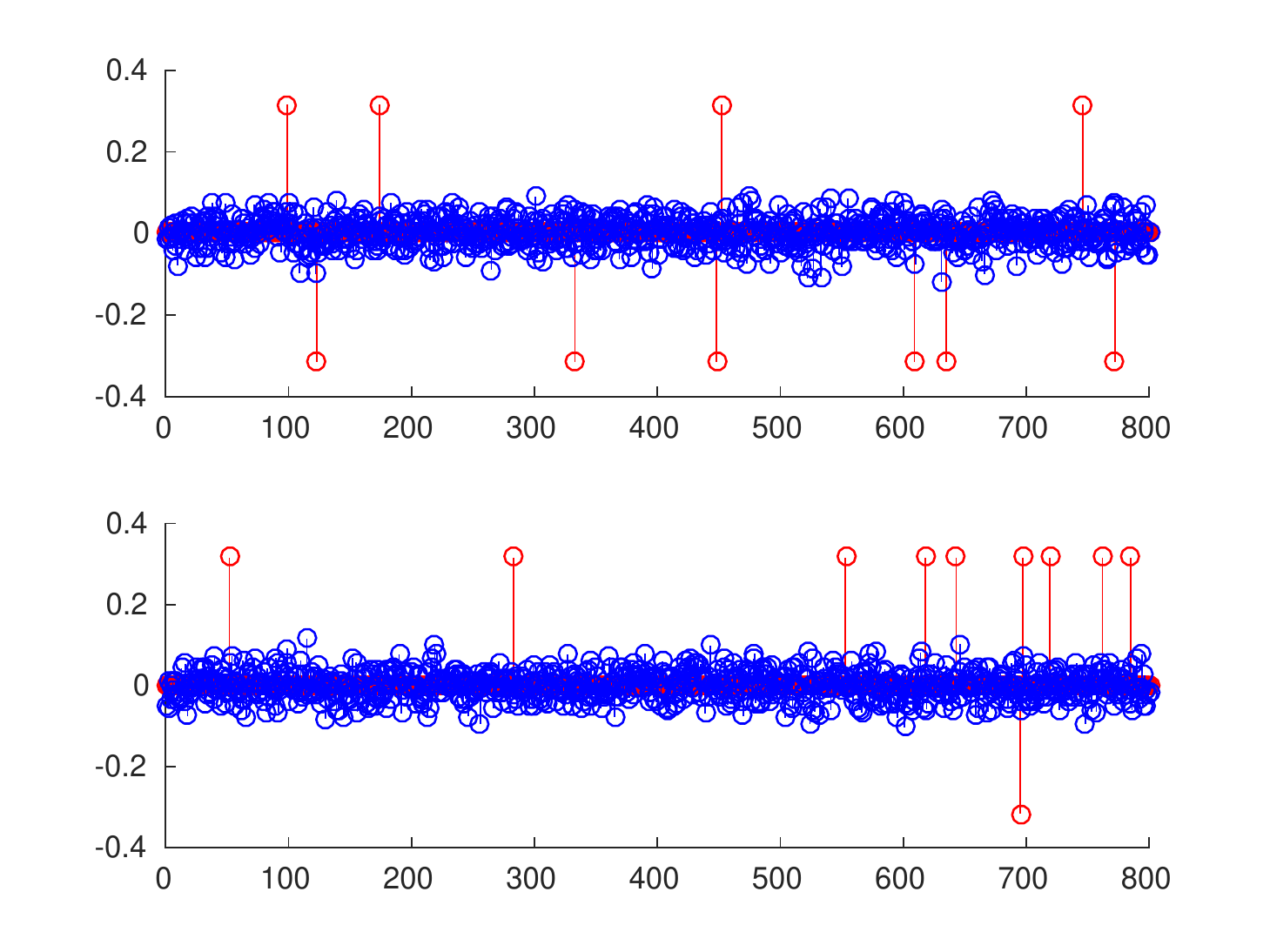}
\caption{\textit{Plot of Estimated $\hat u$, $\hat v$ from singular value decomposition (blue) and true $u$, $v$ (red)}, The number of observations $n = 1000$, with $p = 800, q = 800$. Estimated $u$ and $v$ using singular value decomposition of the transformed estimated covariance matrix are not good estimated of the true $u$ and $v$. The results are wrong and not sparse. This is an indication that we need more samples to estimate the canonical vectors. }
\label{svduv}
\end{figure}

\begin{figure}
\centering
\includegraphics[height = 6 cm]{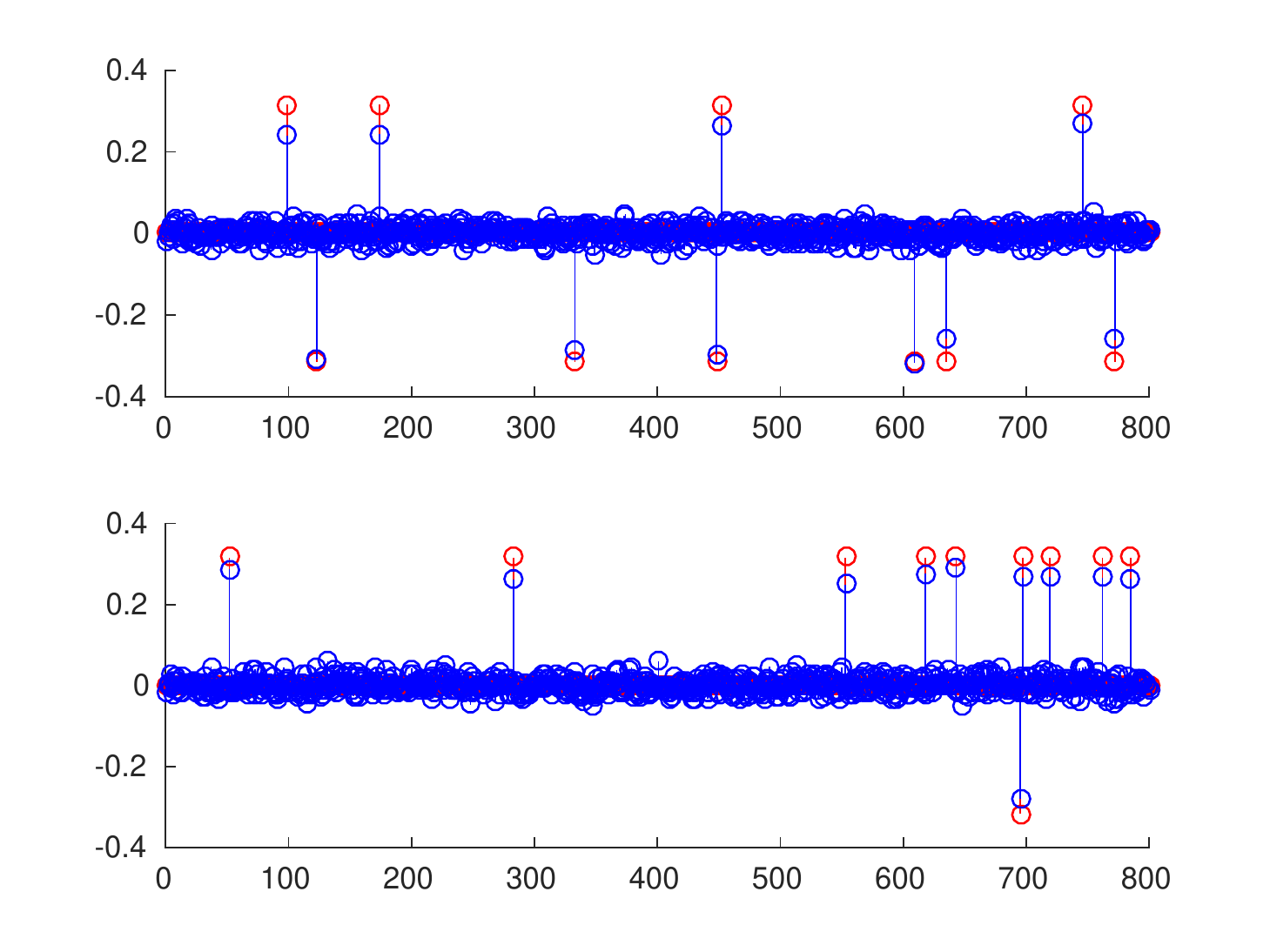}
\caption{\textit{Plot of Estimated $u$, $v$ from singular value decomposition (blue) and true $u$, $v$ (red)}, The number of observations $n = 1000$, with $p = 800, q = 800$. Estimated $u$ and $v$ using singular value decomposition of the transformed estimated covariance matrix are not good estimated of the true $u$ and $v$. The results are wrong and not sparse. This is an indication that we need more samples to estimate the canonical vectors. }
\label{svduvmore}
\end{figure}

\begin{figure}
\centering
\includegraphics[scale = 0.8]{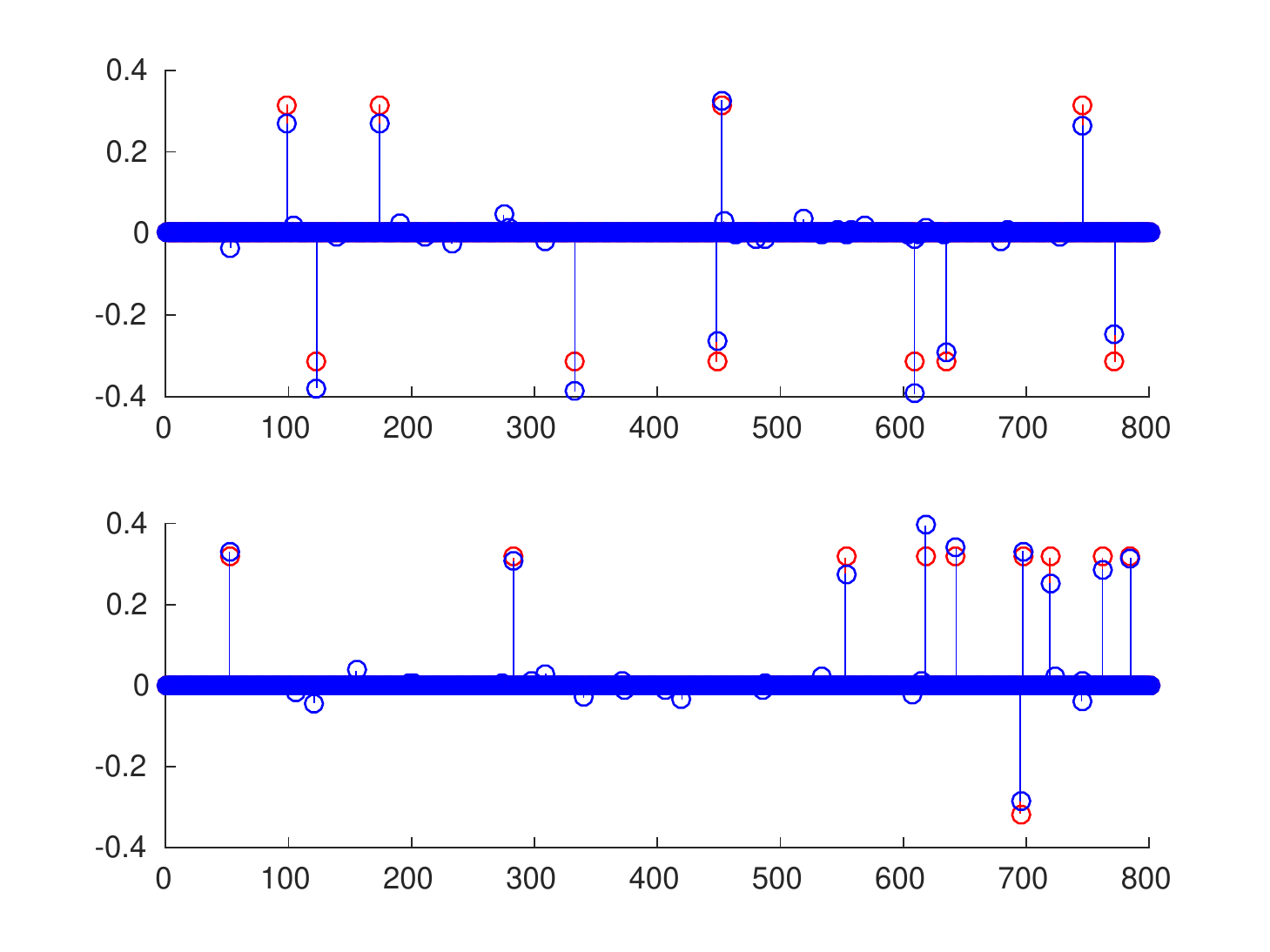}
\caption{\textit{Plot of Estimated $u$, $v$ from our method (blue) and true $u$, $v$ (red)}, The number of observations is $n=400$, with $p = 800, q = 800$. Note that we use less samples than the results of the \autoref{svduv}. We can successfully recover the correct support using our method. }
\label{ouruv}
\end{figure}
\end{example}

\subsection{A detailed Comparison}
To further illustrate the accuracy of our methods, we compare our methods with the methods proposed by \citet{sparseig} using the plot of scaled sample size versus estimation error. Here we choose the same set up with their setup since their method performed the best in comparison with PMA. The data was simulated as follows: 
\begin{align*}
\rho = 0.9, u_{j} = \frac{1}{\sqrt{5}}, v_j = \frac{1}{\sqrt{5}} \text{ for } j = 1,6,11,16,21.
\end{align*}
And $\Sigma_x$ and $\Sigma_y$ are block diagonal matrix with five blocks, each of dimension $d/5 \times d/5$, where the $(j,j')$ th element of each block takes value $0.7^{|j-j'|}$. The result is done for $p_u = 300$, $p_v = 300$ and average over 100 simulations. 

Though the set up of our simulation is the same with \citet{sparseig}, we would like to investigate when the rescaled sample size is small, i.e., when the number of samples is small. 
As shown in \autoref{compareuv}, our method outperforms their method. 
\begin{figure}
\includegraphics[scale=0.5]{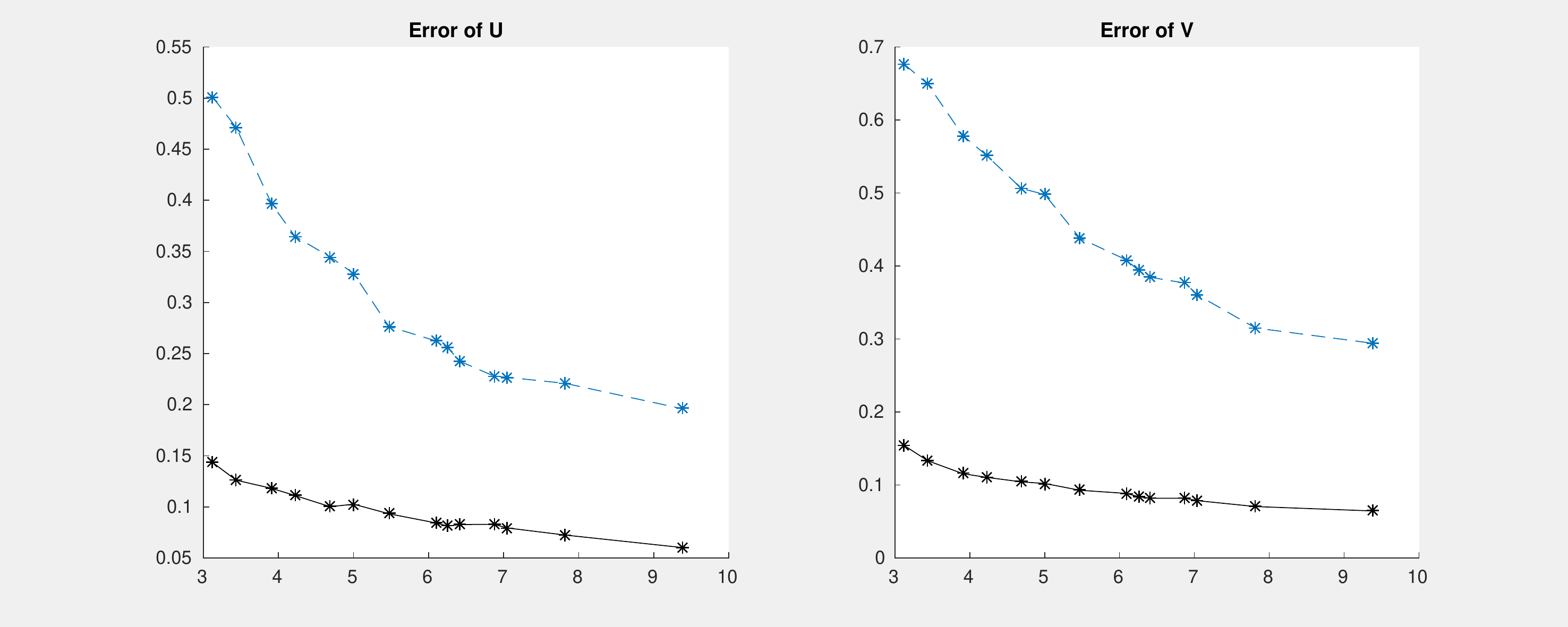}
\caption{A comparison between performance of our method and method proposed by \citet{sparseig}. The left panel is the $\loss(\hat u)$ versus rescaled sample size $n/s\log(d)$, and the right panel is the $\loss(\hat v)$ versus rescaled sample size $n/s\log d$.  Blue line is the result of \citet{sparseig} and the black line is the result of our method. }
\label{compareuv}
\end{figure}

\section{Discussion and future work}
\label{sec: discussion}
We proposed a sparse canonical correlation framework and show how to solve it efficiently using ADMM and TFOCS. We presented different simulation scenarios and showed our estimates are more sparse and accurate. Though our formulation is non-convex, global solutions are often obtained, as seen among simulated examples.  We are currently working on some applications on real data sets. 
\bibliography{Untitled}

\begin{thebibliography}{11}
\providecommand{\natexlab}[1]{#1}
\providecommand{\url}[1]{\texttt{#1}}
\expandafter\ifx\csname urlstyle\endcsname\relax
  \providecommand{\doi}[1]{doi: #1}\else
  \providecommand{\doi}{doi: \begingroup \urlstyle{rm}\Url}\fi

\bibitem[Baik et~al.(2005)Baik, Arous, P{\'e}ch{\'e}, et~al.]{Bphase}
Jinho Baik, G{\'e}rard~Ben Arous, Sandrine P{\'e}ch{\'e}, et~al.
\newblock Phase transition of the largest eigenvalue for nonnull complex sample
  covariance matrices.
\newblock \emph{The Annals of Probability}, 33\penalty0 (5):\penalty0
  1643--1697, 2005.

\bibitem[Becker et~al.(2011)Becker, Cand{\`e}s, and Grant]{TFOCS}
Stephen~R. Becker, Emmanuel~J. Cand{\`e}s, and Michael~C. Grant.
\newblock Templates for convex cone problems with applications to sparse signal
  recovery.
\newblock \emph{Mathematical Programming Computation}, 3\penalty0 (3):\penalty0
  165, 2011.
\newblock ISSN 1867-2957.
\newblock \doi{10.1007/s12532-011-0029-5}.
\newblock URL \url{http://dx.doi.org/10.1007/s12532-011-0029-5}.

\bibitem[{Chen} et~al.(2013){Chen}, {Gao}, {Ren}, and {Zhou}]{thresholding}
M.~{Chen}, C.~{Gao}, Z.~{Ren}, and H.~H. {Zhou}.
\newblock {Sparse CCA via Precision Adjusted Iterative Thresholding}.
\newblock \emph{ArXiv e-prints}, November 2013.

\bibitem[d'Aspremont et~al.(2007)d'Aspremont, Ghaoui, Jordan, and
  Lanckriet]{sparsePCA}
Alexandre d'Aspremont, Laurent~El Ghaoui, Michael~I. Jordan, and Gert R.~G.
  Lanckriet.
\newblock A direct formulation for sparse pca using semidefinite programming.
\newblock \emph{SIAM Review}, 49\penalty0 (3):\penalty0 434--448, 2007.
\newblock \doi{10.1137/050645506}.

\bibitem[{Gao} et~al.(2014){Gao}, {Ma}, and {Zhou}]{sparseCCA}
C.~{Gao}, Z.~{Ma}, and H.~H. {Zhou}.
\newblock {Sparse CCA: Adaptive Estimation and Computational Barriers}.
\newblock \emph{ArXiv e-prints}, September 2014.

\bibitem[Hotelling(1936)]{HOTELLING01121936}
HAROLD Hotelling.
\newblock Relations between two sets of variates.
\newblock \emph{Biometrika}, 28\penalty0 (3-4):\penalty0 321--377, 1936.
\newblock \doi{10.1093/biomet/28.3-4.321}.
\newblock URL \url{http://biomet.oxfordjournals.org/content/28/3-4/321.short}.

\bibitem[Johnstone and Lu(2009)]{johnstone}
Iain~M. Johnstone and Arthur~Yu Lu.
\newblock On consistency and sparsity for principal components analysis in high
  dimensions.
\newblock \emph{Journal of the American Statistical Association}, 104\penalty0
  (486):\penalty0 682--693, 2009.
\newblock \doi{10.1198/jasa.2009.0121}.
\newblock URL \url{http://dx.doi.org/10.1198/jasa.2009.0121}.
\newblock PMID: 20617121.

\bibitem[Parikh and Boyd(2014)]{ADMM}
Neal Parikh and Stephen Boyd.
\newblock Proximal algorithms.
\newblock \emph{Found. Trends Optim.}, 1\penalty0 (3):\penalty0 127--239,
  January 2014.
\newblock ISSN 2167-3888.
\newblock \doi{10.1561/2400000003}.
\newblock URL \url{http://dx.doi.org/10.1561/2400000003}.

\bibitem[Paul(2007)]{paul2007}
Debashis Paul.
\newblock Asymptotics of sample eigenstructure for a large dimensional spiked
  covariance model.
\newblock \emph{Statistica Sinica}, pages 1617--1642, 2007.

\bibitem[{Tan} et~al.(2016){Tan}, {Wang}, {Liu}, and {Zhang}]{sparseig}
K.~M. {Tan}, Z.~{Wang}, H.~{Liu}, and T.~{Zhang}.
\newblock {Sparse Generalized Eigenvalue Problem: Optimal Statistical Rates via
  Truncated Rayleigh Flow}.
\newblock \emph{ArXiv e-prints}, April 2016.

\bibitem[Witten et~al.(2009)Witten, Tibshirani, and Hastie]{Witten}
Daniela~M. Witten, Robert Tibshirani, and Trevor Hastie.
\newblock A penalized matrix decomposition, with applications to sparse
  principal components and canonical correlation analysis.
\newblock \emph{Biostatistics}, 10\penalty0 (3):\penalty0 515--534, 2009.
\newblock \doi{10.1093/biostatistics/kxp008}.
\newblock URL
  \url{http://biostatistics.oxfordjournals.org/content/10/3/515.abstract}.

\end{thebibliography}
\newpage
\section{Appendix}
\subsection*{Detailed derivations for linearized ADMM}
The augmented Lagrangian form of \ref{foru} is 
\begin{align*}
L(u,z, \xi ) =  -u^TX^TYv + \tau_1\|u\|_1 + \Ind\{ \|z\|_2 \leq 1\}+\phi^T(Xu-z) + \frac{\rho}{2}\|Xu-z\|_2^2 . 
\end{align*}
Thus, the updates of variables are solved through
\begin{align*}
u^{k+1}&=  \Argmin_u\{-u^TX^TYv + \tau_1\|u\|_1 + {\phi^k}^T(Xu-z^k) + \frac{\rho}{2}\|Xu-z^k\|_2^2\}\\
z^{k+1} &=\Argmin_z\{\Ind\{ \|z\|_2 \leq 1\} +\phi^T(Xu^{k+1} - z)+ \frac{\rho}{2}\|Xu^{k+1}-z\|_2^2\}\\
\phi^{k+1} &= \phi^k + \rho(Xu^{k+1} - z^{k+1})
\end{align*}
Now, we let $\xi^k = \frac{\phi^k}{\rho}$ and add some constants, and we get 

\begin{align*}
u^{k+1}&=  \Argmin_u\{-u^TX^TYv + \tau_1\|u\|_1 + \rho{\xi^k}^T(Xu-z^k) + \frac{\rho}{2}\|Xu-z^k\|_2^2 + \frac{\rho}{2}\|\xi\|_2^2\}\\
z^{k+1} &= \Argmin_z\{\Ind\{ \|z\|_2 \leq 1\} +\rho\xi^T(Xu^{k+1} - z)+ \frac{\rho}{2}\|Xu^{k+1}-z\|_2^2+ \frac{\rho}{2}\|\xi\|_2^2\}\\
\xi^{k+1} &= \xi^k + (Xu^{k+1} - z^{k+1}).
\end{align*}

Therefore, we have 

\begin{align*}
u &\leftarrow \Argmin_u\{-u^TX^TYv + \tau_1\|u\|_1 + \frac{\rho}{2}\|Xu-z+\xi\|_2^2\}\\
z &\leftarrow \Argmin_z\{\Ind\{ \|z\|_2 \leq 1\} + \frac{\rho}{2}\|Xu-z+\xi\|_2^2\}\\
\xi &\leftarrow \xi + (Xu - z). 
\end{align*}

The linearized ADMM replace the quadratic term by a linear term in order to speed up: 

\begin{align*}
u &\leftarrow \Argmin_u\{-u^TX^TYv + \tau_1\|u\|_1 + \rho(X^TXu^k - X^Tz^k)^Tu + \frac{\mu}{2}\|u-u^k\|_2^2 \}\\
z &\leftarrow \Argmin_z\{\Ind\{ \|z\|_2 \leq 1\} + \frac{\rho}{2}\|z-Xu^{k+1}+\xi^k\|_2^2\}\\
\xi &\leftarrow \xi + (Xu^{k+1} - z^{k+1}). 
\end{align*}
Let $\rho = 1/\lambda$, and $\mu = \frac{1}{\mu}$,  we get 
\begin{align*}
u &\leftarrow \Argmin_u\{-u^TX^TYv + \tau_1\|u\|_1 +\frac{1}{\lambda}(X^TXu^k - X^Tz^k)^Tu + \frac{1}{2\mu}\|u-u^k\|_2^2 \}\\
z &\leftarrow \Argmin_z\{\Ind\{ \|z\|_2 \leq 1\} + \frac{\rho}{2}\|z-Xu^{k+1}+\xi^k\|_2^2\}\\
\xi &\leftarrow \xi + (Xu^{k+1} - z^{k+1}). 
\end{align*}
For the first minimization problem, after some simple algebra, we can get: 

\begin{align*}
u^{k+1} & = \Argmin_u\{-u^TX^TYv + \tau_1\|u\|_1 + \frac{1}{2\mu}\|u - (u^k - \frac{\mu}{\lambda}(X^T(Xu^k - z^k + \xi^k))\|_2^2\}
\end{align*}
Therefore, our detailed updates are: 
\begin{align*}
u^{k+1} &\leftarrow \prox_{\mu f}(u^k - \frac{\mu}{\lambda}X^T(Xu^k - z^k + \xi^k)) \\
z^{k+1} &\leftarrow \prox_{\lambda g}(Xu^{k+1} + \xi^k)\\
\xi^{k+1} &\leftarrow \xi^k  + Xu^{k+1} - z^{k+1}.
\end{align*}

The analytic proximal mapping of $f$ and $g$ can be easily derived: $f(x)$ involves soft threshold and $g(x)$ is projection to the convex set (a norm ball):
\begin{align*}
\prox_{\mu f}(x) = \begin{cases}  x + \mu c  - \mu \tau & \text{ if } x+\mu c > \mu \tau\\
x + \mu c + \mu \tau  &\text{ if } x+\mu c < -\mu \tau\\ 0  &\text{ else }\end{cases}
\end{align*}
\begin{align}
\prox_{\lambda g}(x) = \begin{cases} x &\text{ if } \|x\|_2 \leq 1\\
\frac{x}{\|x\|_2} &\text{ else}
\end{cases}
\end{align}
where $ c = X^TYv$ (the gradient of the objective function with respect to one canonical vector while fixing the other canonical vector).

\end{document}